\documentclass[lettersize,journal]{IEEEtran}
\usepackage{amsmath,amsfonts}
\usepackage{algorithmic}
\usepackage{algorithm}
\usepackage{array}
\usepackage{subfigure}
\usepackage{textcomp}
\usepackage{stfloats}
\usepackage{url}
\usepackage{verbatim}
\usepackage{graphicx}
\usepackage{cite}
\usepackage{xcolor}
\usepackage{booktabs}
\usepackage{setspace}
\usepackage[colorlinks=true,linkcolor=blue,citecolor=blue]{hyperref}%

\newtheorem{thm}{Theorem}

\newtheorem{rem}{Remark}
\newtheorem{proof}{Proof}
\newtheorem{assu}{Assumption}

\newcommand{\tr}{\mathop{\mathrm{tr}}}

\begin{document}

\title{Learning and Current Prediction of PMSM Drive via Differential Neural Networks}

\author{Wenjie~Mei,~\IEEEmembership{Member,~IEEE,}
        Xiaorui~Wang,
        Yanrong~Lu,
        Ke~Yu
        and Shihua~Li,~\IEEEmembership{Fellow,~IEEE} 
        %and Xinghuo~Yu,~\IEEEmembership{Fellow,~IEEE}
	\thanks{Wenjie Mei, Ke Yu, and Shihua Li are with the School of Automation and the Key Laboratory of MCCSE of the Ministry of Education, Southeast University, Nanjing 210096, China. Xiaorui Wang and Yanrong Lu are with the School of Electrical and Information Engineering, Lanzhou University of Technology, Lanzhou 730050, China. } %Xinghuo Yu is with the School of Engineering, Royal Melbourne Institute of Technology (RMIT) University, Melbourne, VIC 3000, Australia. }% (Corresponding author: Shihua Li, e-mail: lsh@seu.edu.cn)

}  

\markboth{Journal of \LaTeX\ Class Files,~Vol.~14, No.~8, August~2021}%
{Shell \MakeLowercase{\textit{et al.}}: A Sample Article Using IEEEtran.cls for IEEE Journals}


\maketitle

\begin{abstract}

Learning models for dynamical systems in continuous time is significant for understanding complex phenomena and making accurate predictions. This study presents a novel approach utilizing differential neural networks (DNNs) to model nonlinear systems, specifically permanent magnet synchronous motors (PMSMs), and to predict their current trajectories. The efficacy of our approach is validated through experiments conducted under various load disturbances and no-load conditions. The results demonstrate that our method effectively and accurately reconstructs the original systems, showcasing strong short-term and long-term prediction capabilities and robustness. This study provides valuable insights into learning the inherent dynamics of complex dynamical data and holds potential for further applications in fields such as weather forecasting, robotics, and collective behavior analysis.
\end{abstract}


\section{Introduction}
\IEEEPARstart{T}{he} recent advancements in machine learning techniques have significantly expanded opportunities for the automated extraction of knowledge and patterns from data. Deep learning models, renowned for their robust approximation capabilities, have achieved substantial progress across various fields. This progress has inspired researchers to investigate the application of neural networks for uncovering the dynamical laws underlying observational data \cite{lsh, NN2, raissi2019physics}. This exploration represents a novel paradigm for enhancing the learning and prediction of complex physical models.

Differential neural networks (DNNs) \cite{DNN,sarigul2019differential,chen2018neural} have been regarded as a powerful deep learning methodology for modeling dynamical systems, especially those in continuous time. These networks conceptualize the system state as a continuously evolving latent variable and utilize neural networks to approximate the changing rate of this state. By numerically integrating from the initial state, one can derive the system trajectory. This technique has demonstrated exceptional performance in both modeling and prediction across a variety of domains. Beyond their capacity to handle multimodal data, including text, images, and videos \cite{park2021vid, zhang2022metanode}, DNNs have proven effective in various engineering applications \cite{kidger2020neural, de2019gru, rubanova2019latent}, bio-medicine \cite{qian2021integrating,  bram2024low}, and particularly in engineering for modeling dynamical systems with unknown components \cite{norcliffe2020second, lou2020neural}.

This brief focuses on the challenge of learning and predicting currents in permanent magnet synchronous motor (PMSM) to ensure safety guarantees for these currents during the prediction period. PMSMs have gained significant attention in industrial applications, such as aerospace, robotic manipulators, and electric vehicles \cite{PMSM2,liu2016research}, due to their notable advantages, including simple structure, high power density, high torque-to-inertia ratio, and efficiency \cite{kk}. The PMSM is inherently a nonlinear system, characterized by complex parameter interactions and intertwined disturbances. To this end, different neural networks are proposed and applied to approximate the unknown dynamics of  PMSM model (see, for example, \cite{NN3, NN31, NN32}).

However, most existing neural network methods take discrete-time architectures and often assume that they can perfectly learn continuous-time characteristics, deviating from the inherent continuous nature commonly present in the real physical world, such as rigid body systems, chemical processes, and electrical circuits. The continuous dynamic property not only embodies rich prior knowledge but also has a decisive influence on the long-term behavior of the dynamics.  To ensure the safety of  PMSM operation, it is generally necessary to set the current limit in the software. In practice, current sensors can only acquire information about the current in real time, but cannot predict the current. Therefore, it is of significance to forecast the current evolution. 

%The suitable way of incorporating physical constraints into the NODE learning process, so that the learned model can adhere to physical laws, is a significant issue that urgently needs to be addressed.

This work addresses the challenge of achieving high-accuracy predictions for PMSMs to ensure safety guarantees. The primary theoretical contribution of this study is the development of practical learning laws for updating the weight matrices of our proposed DNN, enabling it to accurately approximate the ideal model. Additionally, the experimental results demonstrate that our model outperforms other widely used neural networks, such as convolutional neural networks and transformers, in the tested scenarios, which include no-load conditions and various types of load disturbances.
 
The remainder of this brief is organized as follows. Section \ref{sec:Preliminaries} presents the formulation of the dynamics of PMSMs, as well as the research objective. The main theoretical results are presented in Section \ref{sec:dnn}. The experiments on the prediction of currents are given in Section \ref{sec:experiment}. The conclusion is drawn in Section \ref{sec:conclusion}, along with limitations and future work.

\section{Preliminaries}\label{sec:Preliminaries}
%\subsection{PMSM Mathematical Model}
The dynamic  mathematical model of surface-mounted PMSM in the 
synchronous rotational $dq$ reference frame can be presented as
\begin{equation}\label{eq:PMSM}
\left\{
\begin{aligned}
   &\dot{i}_d\!=\!\frac{1}{L_s}\left(u_d-R_s{i}_{d}+\omega_ei_q\right)+\!\Delta \theta_d(i_d,i_q), \\
   &\dot{i}_q\!=\!\frac{1}{L_s}\left[u_q-R_s{i}_{q}-\omega_e(L_si_d+\psi_f)\right] +\!\Delta \theta_q(i_d, i_q),\\
   &\dot{\omega}_m=\frac{3\psi_fn_p}{2J}i_q-\frac{{T}_{L}}{J}-\frac{B}{J}\omega_m,
\end{aligned}
\right.
\end{equation}
where $R_s$ is the stator resistance, $L_s$ is the inductance, $n_p$ is the number of pole pair, $\psi_f$ is the permanent-magnet (PM) flux linkage, $i_{d}$, $i_q$, $u_{d}$ and $u_q$ represent the stator phase currents and terminal voltages in the $dq$-axis, respectively.  $\omega_e$ is the electrical angular velocity, $\omega_m=\omega_e/n_p$ is the mechanical angular velocity. The parameters $J$, $B$, $T_e$, and $T_L$ are the moment of inertia, viscous friction coefficient, electromagnetic torque, and load torque, respectively.  $\Delta \theta_d(\cdot,\cdot)$ and $\Delta \theta_q(\cdot,\cdot)$ are the cumulative uncertainties, which can lead to inaccurate system modeling as well as 
poor control performance.

%\subsection{Research Objective}
In this work, our objective is to achieve high prediction accuracy for the currents $i_d$  and $i_q$ using the proposed DNN model, thereby enabling precise forecasting of their future evolution. This capability is vital for assessing the operational safety of PMSMs, addressing a critical demand in various applications, which constitutes the primary motivation behind our work.

\section{Differential neural networks for Learning and Prediction} \label{sec:dnn}
Let us consider the following simple DNN \cite{mei2024controlsynth,mei2024annular}, which are assumed to have abilities to learn the dynamics~\eqref{eq:PMSM}: 
\begin{equation} \label{eq:DDNN}
    \dot{x} =  a_0 x + (a_1^*)^\top f_1(s_1^* x) + (a_2^*)^\top f_2(s_2^* x) , \quad x_0 := x(0),
\end{equation}
where $x(t) \in \mathbb{R}^2$ is the state vector, $a_0 \in \mathbb{R}^{2 \times 2}$ is a prescribed deterministic matrix, $a_1^*, a_2^* \in \mathbb{R}^{z \times 2}$ are the weight matrices of the output layer, the functions $f_1, f_2 \colon \mathbb{R}^{r} \to \mathbb{R}^{z}$ are suitably selected activation functions, and $s_1^*, s_2^* \in \mathbb{R}^{r \times 2}$. 

We propose the following DNN to identify the neural networks~\eqref{eq:DDNN}: 
\begin{equation} \label{eq:id_nn}
    \dot{\hat{x}} =  a_0 \hat{x} + a_1^\top f_1(s_1 \hat{x}) + a_2^\top f_2(s_2 \hat{x}) , \quad \hat{x}_0 := \hat{x}(0).
\end{equation}

To present the main result of this work, some necessary assumptions are given as follows. 

\begin{assu} \label{assum:lip}
    The activation functions $f_i$ satisfy the Lipschitz continuity conditions, \emph{i.e.}, there exist positive numbers $L_{f_i} < +\infty$ such that
    \[
    \| f_i(c_1) - f_i(c_2) \|^2 \leq L_{f_i} \| c_1-c_2\|^2
    \]
    for all $i \in \{ 1,2\}$ and $c_1,c_2 \in \mathbb{R}^r$. 
\end{assu}

\begin{assu} \label{assum:weight}
    The norms of the weight matrices  $a_i^*, s_i^*$ are bounded: 
    \[
    \| a_i^* \|^2 \leq \overline{a_i^* }< +\infty, \quad \| s_i^* \|^2 \leq \overline{s_i^* }< +\infty, \quad \forall i \in \{ 1,2\}.
    \] 
\end{assu}

We are ready to give the following theorem, which involves asymptotic convergence of the state identification error (between the DNN~\eqref{eq:id_nn} and the model~\eqref{eq:DDNN}) under our proposed learning laws.

\begin{thm} \label{thm:error}
Suppose that the neural networks~\eqref{eq:DDNN} can approximate the dynamics~\eqref{eq:PMSM} well and let the DNN identify~\eqref{eq:DDNN} with the learning laws  
\begin{align}
\dot{a}_i  & = c_i p (\mathcal{D}_{f_i})  \tilde{s}_i \hat{x}  e^\top - c_i p f_i(s_i\hat{x}) e^\top, \nonumber \\
\dot{s}_i  & = - d_i  p (\mathcal{D}_{f_i})^\top a_i e  \hat{x}^\top, \quad i \in \{1,2\} \label{eq:learning_law},
\end{align}
where $p, c_i, d_j$ are positive scalars, $e = \hat{x} - x$, $\tilde{s}_i = s_i - s_i^*$, and $\mathcal{D}_{f_i} = \frac{\partial f_i(\cdot)}{\partial (\cdot)}$. 
If there exists a scalar $C>0$ such that $p>0$ (the same as that in Equation~\eqref{eq:learning_law}) is a solution of the inequality 
$ 
p^2\ell I_2 + 2p a_0 + \beta I_2 \leq -CI_2
$
holds true, where $\ell:= \sum_{i=1}^2 \overline{a_i^*}, \; \beta := \sum_{i=1}^2 L_{f_i} \overline{s_i^*} >0$, and $I_2$ is the $2 \times 2$ identity matrix. Then, the error  $e$ has a locally asymptotic equilibrium point at the origin, \emph{i.e.}, $\lim_{t \to +\infty} e(t) =0$. 

\end{thm}

\begin{proof}
Consider the Lyapunov function
\[
V = pe^\top e + \sum_{i=1}^2 c_i^{-1}  \tr\{\tilde{a}_i^\top \tilde{a}_i\} + \sum_{j=1}^2 d_j^{-1} \tr\{\tilde{s}_j^\top \tilde{s}_j\},
\]
where $p$ is a positive scalar, $\tilde{a}_i = a_i - a_i^*$, and  $\tr$ is the trace operator. Taking the time derivative of $V$, we obtain
\[
\dot{V} = 2 p e^\top \dot{e} + 2 \sum_{i=1}^2 c_i^{-1} \tr\{\tilde{a}_i^\top \dot{a}_i \}  +  2 \sum_{j=1}^2 d_j^{-1} \tr\{\tilde{s}_j^\top \dot{s}_j\},
\]
where
\begin{align}
\dot{e} =  a_0 e & + \Big[a_1^\top f_1(s_1 \hat{x})  - (a_1^*)^\top f_1(s_1^* x)  \Big]  \nonumber \\
 & + \Big[ a_2^\top f_2(s_2 \hat{x}) - (a_2^*)^\top f_2(s_2^* x)   \Big]  \label{eq:error_dot}.
\end{align}
Here, we add and subtract some terms for sequential analysis, shown as follows: 
\begin{align}
& a_i^\top f_i(s_i \hat{x})  - (a_i^*)^\top f_i(s_i^* x) \nonumber \\
= &  \tilde{a}_i^\top f_i(s_i\hat{x}) +  (a_i^*)^\top  \Big[ f_i(s_i\hat{x}) -  f_i(s_i^* \hat{x})\Big] \nonumber  \\
&  +  (a_i^*)^\top  \Big[ f_i(s_i^* \hat{x}) -  f_i(s_i^* x) \Big], \quad \forall i \in \{ 1, 2\}.  \nonumber 
\end{align}
Therefore, based on the above formulations, the relation~\eqref{eq:error_dot} can be rewritten to 
\begin{align}
2pe^\top \dot{e} =& 2 p e^\top a_0 e + 2p e^\top \tilde{a}_1^\top f_1(s_1\hat{x}) + 2p e^\top \tilde{a}_2^\top f_2(s_2 \hat{x}) \nonumber  \\ 
& +  2p e^\top  (a_1^*)^\top  \Big[ f_1(s_1\hat{x}) -  f_1(s_1^* \hat{x})\Big] \nonumber \\ 
& + 2p e^\top  (a_1^*)^\top  \Big[ f_1(s_1^* \hat{x}) -  f_1(s_1^* x) \Big] \nonumber\\
& +  2p e^\top (a_2^*)^\top  \Big[ f_2(s_2\hat{x}) -  f_2(s_2^* \hat{x})\Big]\nonumber \\
& + 2p e^\top (a_2^*)^\top  \Big[ f_2(s_2^* \hat{x}) -  f_2(s_2^* x) \Big]. \nonumber
\end{align}
By Assumptions~\ref{assum:lip} and~\ref{assum:weight}, we see that
\begin{align}
 & 2 p e^\top (a_i^*)^\top  \Big[ f_i(s_i^* \hat{x}) -  f_i(s_i^* x)\Big]  \nonumber  \\ 
 \leq &   p^2 \| e \|^2 \| a_i^* \|^2 + \| f_i(s_i^* \hat{x}) -  f_i(s_i^* x) \|^2   \nonumber  \\
    \leq & \| e \|^2 \Big[ p^2 \overline{ a_i^* } + L_{f_i} \overline{s_i^*} \Big].   \nonumber
\end{align} 
Also, with the Mean Value Theorem, it holds that 
\begin{align}
&p e^\top (a_i^*)^\top  \Big[ f_i(s_i\hat{x}) -  f_i(s_i^* \hat{x}) \Big] \nonumber \\  =  & p e^\top a_i^\top (\mathcal{D}_{f_i}) \tilde{s}_i \hat{x}  \nonumber - p e^\top (\tilde{a}_i)^\top (\mathcal{D}_{f_i})  \tilde{s}_i \hat{x}, 
\end{align} 
where 
$
 p e^\top a_i^\top \mathcal{D}_{f_i} \tilde{s}_i \hat{x}  =  \tr\{\tilde{s}_i^\top (\mathcal{D}_{f_i})^\top a_i e p^\top \hat{x}^\top \},  \;
 p e^\top \tilde{a}_i^\top \mathcal{D}_{f_i} $ $ \tilde{s}_i \hat{x} = \tr\{ \tilde{a}_i^\top \mathcal{D}_{f_i} 
 \tilde{s}_i \hat{x}  p e^\top \}, \;
 pe^\top \tilde{a}_i^\top  f_i(s_i\hat{x})  =  \tr\{   \tilde{a}_i^\top f_i(s_i\hat{x})  $ $  p e^\top \}
$
(the property of trace invariance under cyclic permutations is applied). Therefore, under some useful operations, one can deduce 
\begin{align}
\dot{V} \leq & e^\top \Big( p^2 \ell I_2 + 2p a_0 + \beta I_2 \Big) e   \nonumber \\
& + 2 \sum_{i=1}^2 \tr\{ \tilde{a}_i^\top L_{a_i}\} + 2 \sum_{j=1}^2 \tr\{\tilde{s}_i^\top L_{s_i} \},  \nonumber
\end{align}
where the definitions of $\ell, \beta$ can be found in the statement of Theorem~\ref{thm:error}, and 
$
L_{a_i}  = c_i^{-1} \dot{a}_i - p (\mathcal{D}_{f_i})  \tilde{s}_i \hat{x}  e^\top +   p f_i(s_i\hat{x}) e^\top,  \;
L_{s_i}  = d_i^{-1} \dot{s}_i +  p (\mathcal{D}_{f_i})^\top a_i e  \hat{x}^\top, \quad i \in \{1,2\}.
$
Thus, under the relations~\eqref{eq:learning_law} that make $L_{a_i}$ and $L_{s_i}$ zeros, one can readily obtain 
\[
\dot{V} \leq -C \| e \|^2. 
\]
Then, by Barbalat’s Lemma, it is direct to substantiate that $\lim_{t \to +\infty} e(t) =0$. This completes the proof. 
\end{proof}

\begin{rem}
To ensure the existence of a positive solution \( p > 0 \) in Theorem~\ref{thm:error}, we set the matrix \( a_0 \) to have only real eigenvalues, with \(  \eta \) denoting its maximum eigenvalue. Then, the inequality 
\begin{equation} \label{eq:guarantee}
    p^2 \ell + 2p  \eta + \beta + C \leq 0
\end{equation}
represents the condition to be verified, as it follows from the matrix inequality
$
p^2 \ell I_2 + 2p a_0 + \beta I_2 \leq -C I_2 \quad \Longleftrightarrow \quad p^2 \ell I_2 + 2p a_0 + (\beta + C) I_2 \leq 0 \quad \Longleftrightarrow \quad y^\top \left( p^2 \ell I_2 + 2p a_0 + (\beta + C) I_2 \right) y \leq 0, \; \|y\| = 1
$.
By the solution formula of quadratic
inequalities, one can obtain that the following conditions ensure the existence of \( p > 0 \) that satisfies the inequality~\eqref{eq:guarantee}:
\[
\left\{
\begin{array}{ll} 
\frac{-2 \eta + \sqrt{\Delta}}{2\ell} > 0, & \text{if } \Delta := 4 \eta^2 - 4\ell(\beta + C) > 0, \\
 \eta < 0, & \text{if } \Delta = 0.
\end{array} 
\right.
\]
Here, the parameters \( a_0 \) (its maximum eigenvalue: $\eta$), \( \ell \), and \( \beta \) are dependent on the specific characteristics of the DNN defined in~\eqref{eq:DDNN}, and thus can be adjusted accordingly.

\end{rem}

\section{Experiments} \label{sec:experiment}
\subsection{Experimental Setup}
The setup and configuration of a semi-physical experimental platform for DNN training are depicted in Fig. \ref{fig.8}, which consist of four parts: the motor experimental unit, the hardware drive circuits, the power supply, and the host computer with MATLAB/Simulink and Pycharm. The load motor is mechanically coupled with the tested PMSM, and they are simultaneously driven and controlled by the hardware drive circuits. Furthermore, the driving hardware circuit consists of a DSP motion control board (TMS320F280039C), a signal transfer board (TI DRV8305), and two driver boards.  Besides, the $i_d^{*}=0$ vector-control method for PMSM control is implemented by the MATLAB/Simulink software and automatically converted into the C-code for the main control board to execute. The nominal values of PMSM parameters are listed in Table \ref{tab:table1}. 

The state information of these two motors is sampled by the sampling module in the signal transfer board then sent to the main control board, and finally passed to the Pycharm software for DNN training and prediction. Gaussian process regression was employed for reducing uncertainty of experimental data and enhancing the modeling accuracy. All DNN experiments are conducted on a Windows 11 workstation with RAM 16GB, an Intel Core i9-14900KF CPU, and an NVIDIA GeForce RTX 4090 GPU (24 GB). The Pytorch 1.11.0 framework with Python 3.9 and CUDA 11.3 is adopted for GPU acceleration. 
\begin{table}[!htb]
\centering
\caption{The key parameters of PMSM\label{tab:table1}}
\begin{tabular}{| c | c | c |}
\hline
Parameters& Units &Values\\
\hline
Rated power $P$& W & 200\\
Rated speed $N_r$& rpm & 1500\\
Rated load $T_{L}$& N$\cdot$m & 1.2\\
Rated voltage $U_r$& V & 36\\
Rated current $I_r$& A & 7.5\\
Pole pairs  $N_p$& * & 5\\
Rotational inertia $J$&Kg$\cdot$m$^2$& 6.8$\times$ 10$^{-5}$\\
PM flux linkage $\psi_{f}$&Wb & 0.012\\
Stator resistance $R_s$ &$\Omega$& 0.4\\
Winding inductance $L_s$ &mH& 0.7\\
\hline
\end{tabular}
\end{table}

\begin{figure}[h]
    \centering    
\includegraphics[width=\hsize]{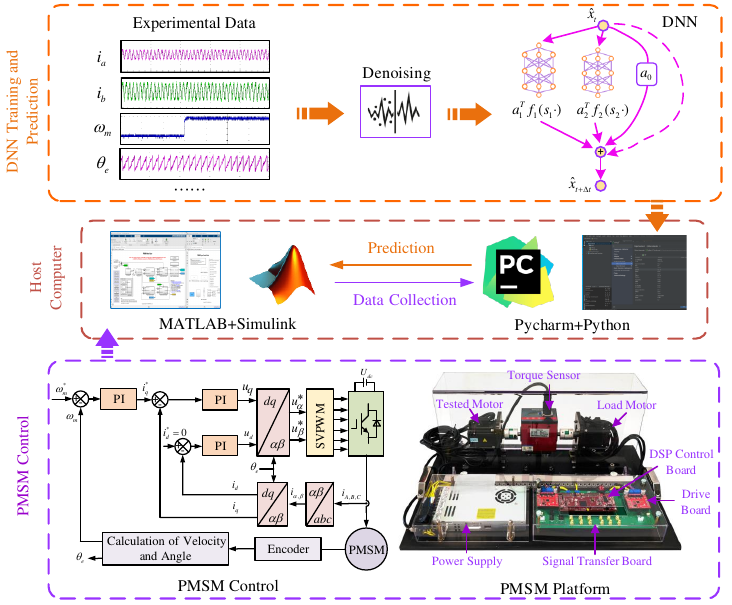}
    \caption{The current prediction framework of PMSMs with DNN.}\label{fig.8}
\end{figure}

\subsection{Test Results}
In this section, we conduct three sets of experiments to thoroughly evaluate the predictive performance of the proposed DNN model. The experimental procedures and corresponding results are detailed below, focusing exclusively on the prediction phase, with no training results presented.

To assess the performance of various commonly used neural network architectures in predicting currents in PMSM, we conducted a comparative analysis involving three models: \textbf{1) DNN}: This is our proposed model, designed to effectively capture nonlinear relations within the data. \textbf{2) Convolutional neural network (CNN)}: A widely adopted architecture, particularly effective for processing grid-structured data, such as images, by learning spatial hierarchies through its convolutional layers \cite{lecun1995convolutional}. \textbf{3) Transformer}: This model leverages self-attention mechanisms to efficiently process sequential data, enabling it to capture long-range dependencies without relying on recurrence, making it well-suited for tasks involving temporal data \cite{vaswani2017attention}.
In these experiments, the DNN consists of 4 fully connected layers, the CNN encompasses three convolutional layers and one linear layer, and the Transformer comprises 2 encoder layers. The learning rate is all set at 0.0001, and the Adam optimizer is employed for training.

\begin{figure}[!htb] 
    \centering    
\subfigure[]{
            \includegraphics[width=0.48\hsize]{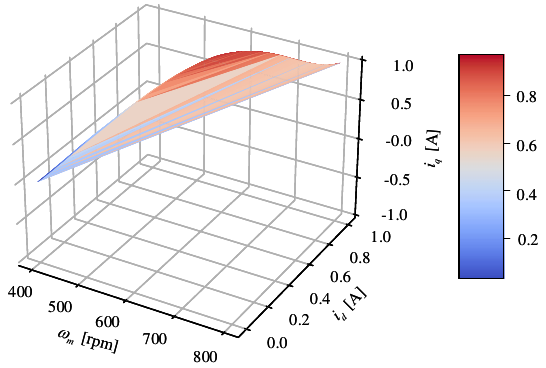}} 
\subfigure[]{
            \includegraphics[width=0.48\hsize]{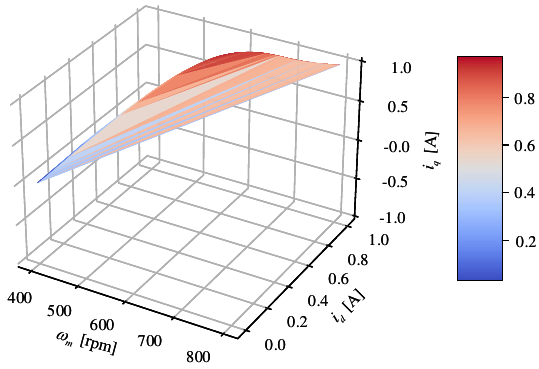}}\\   
\subfigure[]{
           \includegraphics[width=0.48\hsize]{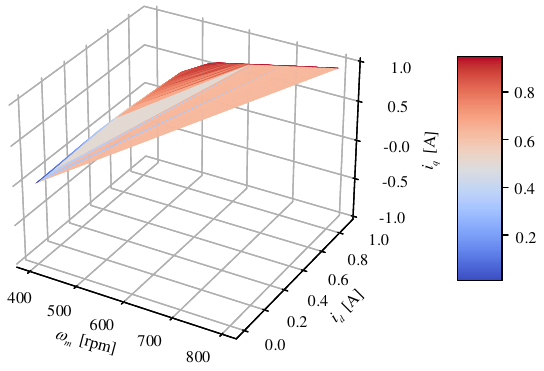}}
\subfigure[]{
           \includegraphics[width=0.48\hsize]{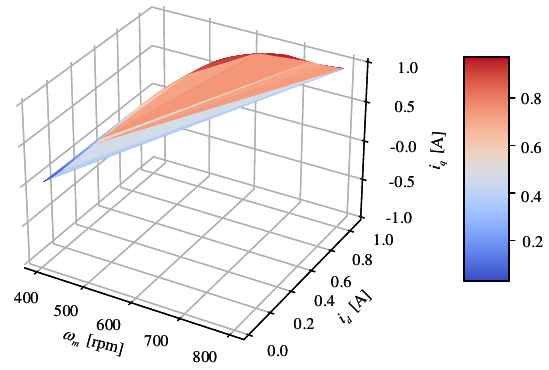}}\\ 
           \vspace{0.1cm}  %调整表格与下文的垂直距离 
    \caption{Prediction result comparison among  (a) \emph{Ground Truth}. (b) \emph{DNN (Ours)}. (c) \emph{CNN}. (d) \emph{Transformer}. }\label{fig.14}
\end{figure}

\emph{Case I}: Fig.~\ref{fig.14} shows the prediction result comparison among ground truth, DNN, CNN, and Transformer with dynamic speed regulation under no-load condition. The speed reference of PMSM ramps up from 400 rpm to 800 rpm within 0.4 seconds. A total of 400 data points were collected in this 0.4 seconds. These data points were divided into two sets: the initial 300 points were utilized for model training, while the remaining 100 points were reserved for prediction purposes. This approach was designed to evaluate the predictive capabilities of the models under investigation. To evaluate the predictive performance of the neural networks, the following metrics were selected: Mean Absolute Error (MAE), Root Mean Squared Error (RMSE), and the Coefficient of Determination (R$^2$). The corresponding results are presented in Table \ref{tab:comparision：case I}. The findings indicate that the DNN outperforms the other models in prediction accuracy, primarily due to its continuous nature, which enhances its precision and adaptability to the dynamics of the PMSM, a system also characterized by continuous-time behavior.
\begin{table}[!t]
\vspace{0cm}  %调整表格与下文的垂直距离
\small
\centering
\caption{Loss/R$^2$ comparison among DNN, CNN, and Transformer}
  \renewcommand{\arraystretch}{1}
  %\label{tab1}
  \begin{tabular}{cccc}
 \toprule[1.2pt]
   \textbf{Model} &  \textbf{MAE}&\textbf{RMSE}&\textbf{R$^2$}\\ 
    \toprule[1.2pt]
    DNN (Ours)  &9.99e-5      &1.98e-4      &0.998\\
    CNN         & 2.81e-4    &4.07e-4       &0.995\\
    Transformer & 2.72e-4    &3.08e-4       &0.997 \\
  \hline 
   \end{tabular} \label{tab:comparision：case I}
  \vspace{0.1cm}  %调整表格与下文的垂直距离
\end{table}

\emph{Case II}: Fig. \ref{fig.15} presents a comparative analysis of the experimental results at 1000 rpm with 1 N$\cdot$m step load disturbance.  A total of 1000 data points were collected over 0.1 seconds, with the first 800 points utilized for model training and the remaining 200 points for prediction. 
The performance of three distinct models (\emph{i.e.}, DNN, CNN, and Transformer) was assessed, and their respective prediction accuracies were compared. The detailed performance metrics, including MAE, RMSE, and R$^2$, are provided in Table \ref{tab:comparison:case II}. These metrics collectively illustrate the superior performance of our proposed DNN model in comparison to the CNN and Transformer models.

Although the R$^2$ value of the DNN is only marginally higher than that of the Transformer, the DNN shows a clear advantage in terms of prediction accuracy, as evidenced by its significantly lower MAE and RMSE values. This indicates that the DNN not only offers a better overall fit but also reduces the prediction error more effectively than the competing models. These results highlight the robustness of our DNN in handling dynamic conditions, especially in the presence of disturbances.
\begin{figure}[!tb] 
    \centering    
\subfigure[]{
            \includegraphics[width=0.48\hsize]{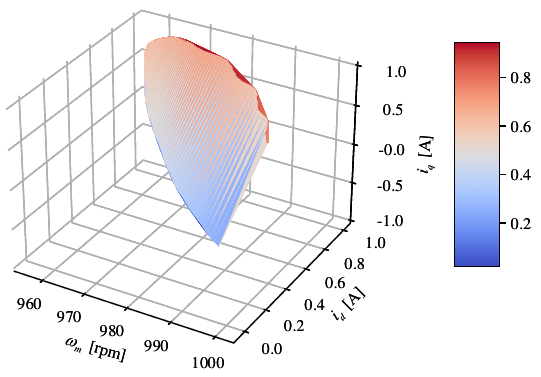}} 
\subfigure[]{
           \includegraphics[width=0.48\hsize]{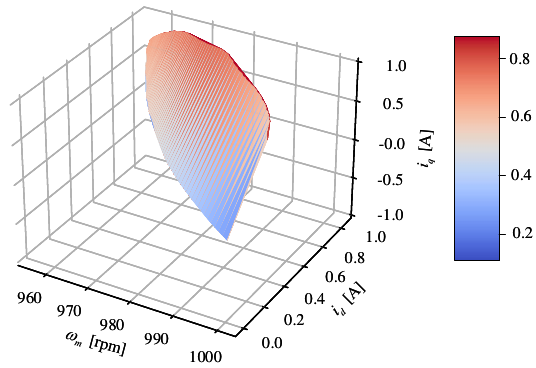}} \\ 
\subfigure[]{
           \includegraphics[width=0.48\hsize]{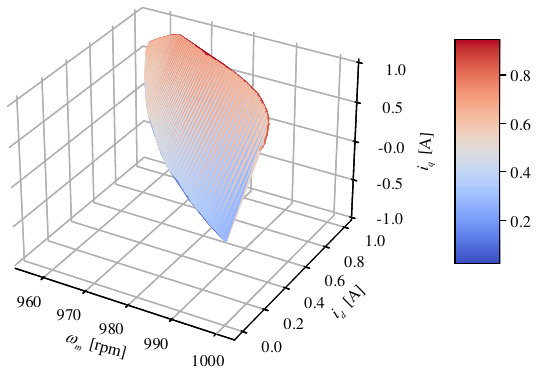}}
\subfigure[]{
            \includegraphics[width=0.48\hsize]{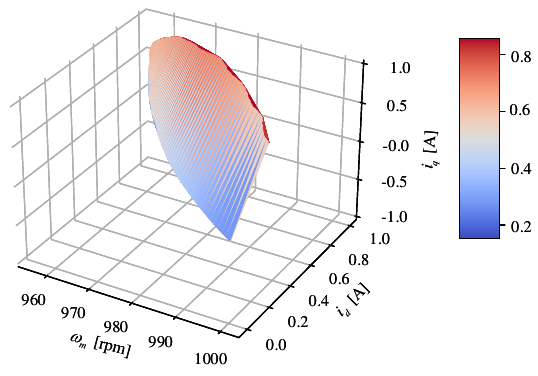}}
    \caption{Prediction comparison among  (a) \emph{Ground Truth}. (b) \emph{DNN (Ours)}.  (c) \emph{CNN}. (d) \emph{Transformer}. }\label{fig.15}
\end{figure} 

\begin{table}[!t]
\vspace{0cm}  %调整表格与下文的垂直距离
\small
\centering
\caption{Loss/R$^2$ comparison among  DNN, CNN, and Transformer}
  \renewcommand{\arraystretch}{1}
  %\label{tab1}
  \begin{tabular}{cccc}
 \toprule[1.2pt]
   \textbf{Model} &  \textbf{MAE}&\textbf{RMSE}&\textbf{R$^2$}\\ 
    \toprule[1.2pt]
    DNN (Ours) & 0.010 &0.014 &0.997\\
    CNN  & 0.063 &0.109 &0.974 \\
    Transformer  & 0.046 &0.067 &0.987 \\
  \hline 
   \end{tabular} \label{tab:comparison:case II}
  \vspace{-0.15cm}  %调整表格与下文的垂直距离
\end{table}

\emph{Case III}:
To assess the prediction robustness of the DNN under various load conditions, which is important for practical applications, we conducted an experiment under different load disturbances at 1000 rpm. The load disturbances are selected as: 1) Step load with 1 N$\cdot$m; 2) Ramp load with slope 1 N$\cdot$m$/$s, and the final value is 1 N$\cdot$m;  3) Sinusoidal load $T_{Lsin}(t) = \sin(20\pi t)$.
% \begin{equation}
% \setlength{\abovedisplayskip}{4pt}
% \begin{array}{ll}
% \text{Step load:} \;\; T_{Ls}(t) = \left\{ 
% \begin{array}{ll}
% 0 & \text{if } t < t_a, \\
% 1 & \text{if } t \geq t_a, 
% \end{array}
% \right. \\[10pt]
% \text{Ramp load:} \; \; T_{Lr}(t) = \left\{ 
% \begin{array}{ll}
% 0 & \text{if } t < t_a, \\
% t - t_a & \text{if } t \geq t_a,
% \end{array}
% \right. \\[10pt]
% \text{Sinusoidal load:} \;\; T_{Lsin}(t) = \sin(20\pi t),
% \end{array}
% \label{eq3-1}
% \end{equation}
% where $t_a$ is the time instant of adding the load. 

Fig. \ref{fig.three_disturbance} illustrates the prediction performance of the DNN under different load disturbances. Task (a), involving a step load disturbance, presents a significant challenge due to the presence of bends in the true current trajectory (phase plane). Nevertheless, our model can predict this behavior with high accuracy. In task (b), the difference between the predicted values and the ground truth is minor, further demonstrating the strong predictive capabilities of the DNN. Task (c) features a sinusoidal load disturbance, which is a common scenario in practical applications. As shown in  Fig. \ref{fig.three_disturbance}(c), the DNN provides satisfactory predictions, with only a noticeable non-overlapping section between the predicted and true trajectories at the lower left corner of the phase plane. The corresponding MAE, RMSE, and R$^2$ values for the tasks are presented in Fig.~\ref{fig.three_disturbance}(d). 

The proposed DNN outperforms the competing models for two main reasons: First, it is formulated in continuous time, which aligns more naturally with the DNN model’s dynamics without the constraints of discretization. Second, the DNN model’s structure, with one linear and two nonlinear terms, closely resembles that of the PMSM. This similarity reduces the modeling gap and mitigates learning challenges.

\begin{figure}[!htb]
    \centering    
\subfigure[]{
           \includegraphics[width=0.48\hsize]{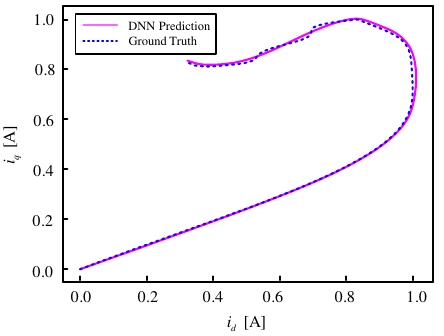}} 
\subfigure[]{
           \includegraphics[width=0.48\hsize]{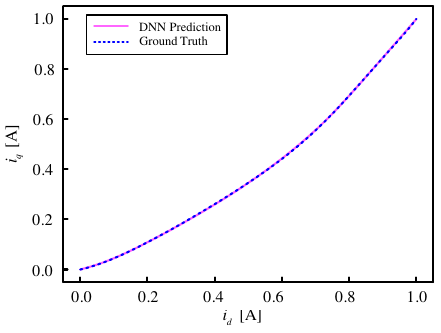}}\\
\subfigure[]{
            \includegraphics[width=0.465\hsize]{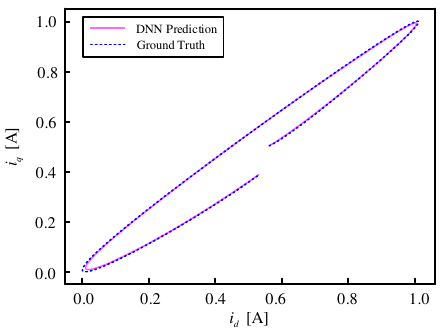}}    
\subfigure[]{
          \includegraphics[width=0.495\hsize]{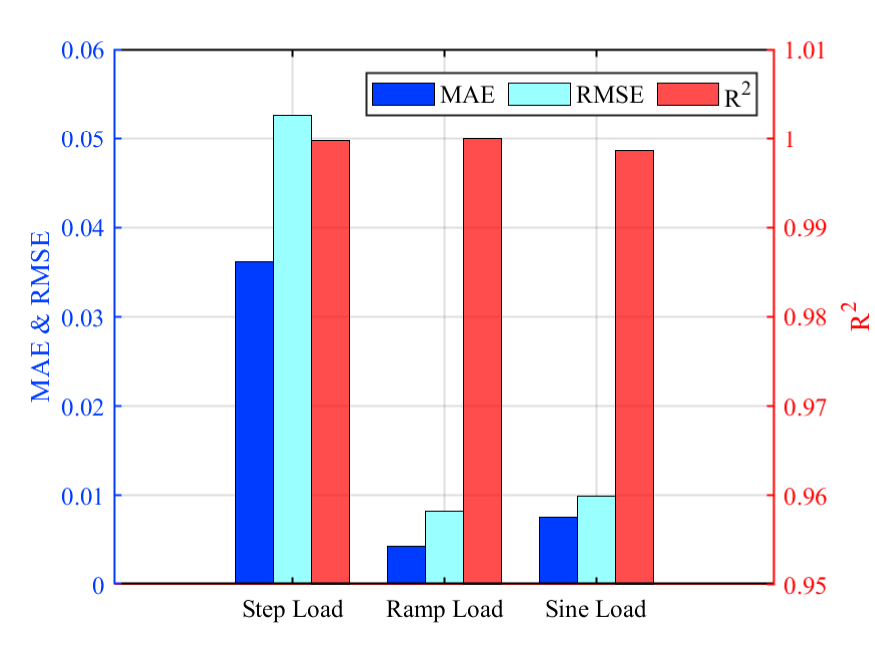}} 
    \caption{DNN prediction results under different load disturbances: (a) \emph{Step load disturbance}. (b) \emph{Ramp load disturbance}. (c) \emph{Sinusoidal load disturbance}. (d) \emph{Loss/R$^2$ values under the three load disturbances}.}\label{fig.three_disturbance}
\end{figure}

\section{Conclusion}\label{Conclusion} \label{sec:conclusion}
In this brief, we introduced a novel architecture of differential neural networks (DNNs) for modeling and predicting the behavior of nonlinear dynamical systems, specifically focusing on permanent magnet synchronous motors (PMSMs). Our approach demonstrates significant efficacy in reconstructing the original systems and provides robust short-term and long-term prediction capabilities. The experimental validation under varied load disturbances and no-load conditions highlights the practical utility and accuracy of our method.

However, several limitations have to be acknowledged. Firstly, the current study primarily focuses on PMSMs, and the generalizability of the proposed model to other nonlinear systems remains an area for future exploration. Additionally, while the approach exhibits robustness in the tested scenarios, the performance under extreme or unforeseen disturbances has not been thoroughly investigated. The computational complexity associated with training the DNN model may also pose challenges in real-time applications.

Future work will address these limitations by extending the methodology to a broader range of real systems and exploring its performance in more diverse and extreme conditions. Furthermore, integrating the proposed approach with advanced techniques in system identification and control could provide deeper insights and broader applicability in many fields. %investigating methods to reduce computational overhead will be important for practical implementations. %Additionally, integrating the proposed approach with advanced techniques in system identification and control could provide deeper insights and broader applicability in main fields.

\bibliographystyle{IEEEtran}
\bibliography{IEEEtrans}

\begin{comment}
\newpage

\section{Biography Section}
If you have an EPS/PDF photo (graphicx package needed), extra braces are
 needed around the contents of the optional argument to biography to prevent
 the LaTeX parser from getting confused when it sees the complicated
 $\backslash${\tt{includegraphics}} command within an optional argument. (You can create
 your own custom macro containing the $\backslash${\tt{includegraphics}} command to make things
 simpler here.)
 
\vspace{11pt}

\bf{If you include a photo:}\vspace{-33pt}
\begin{IEEEbiography}[{\includegraphics[width=1in,height=1.25in,clip,keepaspectratio]{fig1}}]{Michael Shell}
Use $\backslash${\tt{begin\{IEEEbiography\}}} and then for the 1st argument use $\backslash${\tt{includegraphics}} to declare and link the author photo.
Use the author name as the 3rd argument followed by the biography text.
\end{IEEEbiography}

\vspace{11pt}

\bf{If you will not include a photo:}\vspace{-33pt}
\begin{IEEEbiographynophoto}{John Doe}
Use $\backslash${\tt{begin\{IEEEbiographynophoto\}}} and the author name as the argument followed by the biography text.
\end{IEEEbiographynophoto}
\end{comment}

\end{document}